\documentclass[11pt]{article}
\usepackage[utf8]{inputenc}
\usepackage{parskip}
\usepackage[T1]{fontenc}
\usepackage{array}
\usepackage[utf8]{inputenc}
\usepackage[english]{babel}
\usepackage{pgfplots}
\usepackage{wrapfig}
\usepackage{graphicx,wrapfig,lipsum}
\usepackage{float}    
\usepackage{verbatim} 
\usepackage{amsmath, amssymb, amsthm, mathtools}
\usepackage{caption}
\usepackage{subcaption}
\usepackage[sort&compress, square, numbers]{natbib}
\usepackage[colorlinks=true]{hyperref}
\usepackage[dvipsnames]{xcolor}
\hypersetup{colorlinks, linkcolor={red!80!black}, citecolor={blue!80!black}, urlcolor={blue!80!black}}
\usepackage{cleveref}
\usepackage{bookmark}
\usepackage{fullpage}
\usepackage{enumerate}
\usepackage{paralist}
\usepackage{xspace}
\usepackage{bbm}
\usepackage{bm}
\usepackage{url}
\usepackage{lipsum}
\usepackage{algorithm, algorithmic}
\usepackage{color}
\usepackage{microtype}
\usepackage[section]{placeins}
\usepackage{lscape}
\usepackage{multirow}
\usepackage{todonotes}
\usepackage{enumitem}
\usepackage{booktabs}
\usepackage{accents}
\setlist[enumerate,1]{label=(\roman*)}
\makeatletter

\newcommand{\norm}[1]{\left\lVert#1\right\rVert}

\newcolumntype{L}[1]{>{\raggedright\arraybackslash}p{#1}}
\newcolumntype{C}[1]{>{\centering\arraybackslash}p{#1}}
\definecolor{secgray}{gray}{0.88}

\newcommand{\cN}{\mathcal{N}}

\newcommand{\cP}{\mathcal{P}}

\newcommand{\cX}{\mathcal{X}}

\newcommand{\cZ}{\mathcal{Z}}

\newcommand{\bB}{\mathbb{B}}

\newcommand{\bP}{\mathbb{P}}
\newcommand{\bQ}{\mathbb{Q}}
\newcommand{\bR}{\mathbb{R}}

\newtheorem{theorem}{Theorem}

\newtheorem{lemma}{Lemma}

\newtheorem{corollary}{Corollary}
\newtheorem{proposition}{Proposition}

\newtheorem{assumption}{Assumption}

\def \RR {\mathbb{R}}

\def \QQ {\mathbb{Q}}

\def \lip {\mathrm{lip}}
\newcommand{\EE}{\mathbb{E}}
\newcommand{\W}{\mathrm{W}}

\DeclareMathOperator*{\argmax}{argmax}
\DeclareMathOperator*{\argmin}{argmin}

\title{Wasserstein Distributionally Robust Online Learning}

\date{}
\author{
Guixian Chen\\
University of Michigan\\
\texttt{gxchen@umich.edu}\\
\and
Salar Fattahi\\
University of Michigan\\
\texttt{fattahi@umich.edu}
\and 
Soroosh Shafiee\\
Cornell University\\
\texttt{shafiee@cornell.edu}
}

\pgfplotsset{compat=1.18}
\allowdisplaybreaks

\begin{document}

\maketitle

\begin{abstract}
We study online Wasserstein distributionally robust optimization (DRO), where the ambiguity set is a Wasserstein ball centered at the empirical distribution and evolves with each observation. We formulate each round as a zero-sum game and develop an online distributional best-response algorithm that requires only one projected subgradient step per iteration. Combining regret analysis with measure concentration, we establish explicit, instance-dependent, and high-probability convergence rates to the optimal offline DRO value.
\end{abstract}

\section{Introduction}

Statistical learning seeks a decision $x$ in a feasible set $\cX \subseteq \RR^n$ that minimizes the expected loss $\EE_{z \sim \bP^\star}[\ell(x, z)]$, where $\ell: \cX \times \cZ \to \RR$ denotes the loss function and $\bP^\star \in \cP(\cZ)$ is an unknown distribution supported on $\cZ \subseteq \RR^m$.
When $\mathbb{P}^\star$ is inaccessible but a static dataset of $T$ i.i.d. observations $\{ \hat{z}_1, \dots, \hat{z}_T \}$ is available, Empirical Risk Minimization (ERM) approximates this problem using the empirical distribution $\hat{\mathbb{P}}_T := \frac{1}{T} \sum_{t=1}^T \delta_{\hat{z}_t}$, where $\delta_{z}$ denotes the Dirac measure centered at~$z$. By replacing $\mathbb{P}^\star$ with this plug-in estimator, ERM solves the optimization problem
$$
\textstyle \inf_{x \in \mathcal X} \ \left\{ \mathbb{E}_{z \sim \hat{\bP}_T}[\ell(x,z)] = \frac{1}{T} \sum_{t=1}^T \ell(x, \hat{z}_t) \right\}.
$$
When observations arrive sequentially, Online Convex Optimization (OCO) provides algorithms with sublinear regret, which yields convergence guarantees in expected loss under i.i.d. sampling \citep{hazan2022introduction,shalev2025online}.

The expected loss, however, ignores the variability of the loss, which matters in safety-critical applications.
It is also sensitive to training data corrupted by measurement noise or by an adversary \citep{nietert2023outlier, nietert2024robust}.
It also assumes that the test and training distributions are the same, so its performance can degrade under distribution shift \citep{yang2024generalized} and test-time corruption \citep{kurakin2016adversarial, ian15adversarial}.
Wasserstein distributionally robust optimization (DRO) addresses these weaknesses in a unified fashion.
For $p \in [1, \infty)$, the $p$-Wasserstein distance between $\bP, \bQ \in \cP(\cZ)$ is defined as 
$$\textstyle \W_p(\bP, \bQ) = \inf_{\gamma \in \Gamma(\bP, \bQ)} ( \EE_{(z, z') \sim \gamma}[\|z - z'\|^p] )^{1/p},$$ where $\Gamma(\bP, \bQ)$ is the set of probability measures on $\cZ \times \cZ$ with marginals $\bP$ and $\bQ$.
For $\rho > 0$, the Wasserstein ball $\bB^p_\rho(\bP) = \{ \bQ \in \cP(\cZ) \mid \W_p^p(\bQ, \bP) \le \rho \}$ with center $\bP$ and radius $\rho$ contains the distributions obtained from $\bP$ by moving arbitrary mass whose amount is bounded by $\rho$, which models data corruption and distribution shift.
A risk-averse learner aims to solve
\begin{align}
	\label{eq:WDRO}
	\inf_{x \in \cX} \sup_{\bQ \in \bB^p_\rho(\bP^\star)} \EE_{z \sim \bQ}[\ell(x, z)],
\end{align}
that is, to find a decision $x \in \cX$ that minimizes the worst-case expected loss over all distributions in $\bB^p_\rho(\bP^\star)$.
Since $\bP^\star$ is unknown, the data-driven approach of \citep{mohajerin2018data} replaces $\bP^\star$ in \eqref{eq:WDRO} with its empirical analog $\hat \bP_T$ and solves the resulting offline problem.

In many applications such as online recommendation \citep{wen2022distributionally} and portfolio management \citep{costa2023distributionally}, the data arrive as a stream. 
In such settings, re-solving an offline DRO problem after each data point is expensive, and waiting for a large batch delays the decision.
This motivates our central question:
\begin{quote}
	\emph{Can a learner update its decision with one first-order step per observation and still compete with the offline Wasserstein DRO solution \eqref{eq:WDRO}?}
\end{quote}

\subsection{Summary of Contributions}

We treat \eqref{eq:WDRO} as a zero-sum game played over $T$ rounds.
In round $t$, the learner has observed $\hat z_1, \dots, \hat z_t$, which are independently drawn from $\bP^\star$ with the empirical distribution $\hat \bP_t$.
The adversary, or dual player, picks $\bQ_t \in \bB^p_\rho(\hat \bP_t)$, and the learner, or primal player, picks $x_t \in \cX$ and pays the expected loss $\EE_{z \sim \bQ_t}[\ell(x_t, z)]$.
The learner is judged against the offline benchmark \eqref{eq:WDRO}, whose ball is centered at $\bP^\star$ and not at any empirical distribution.
Unlike the online saddle-point problems with a fixed dual set~\citep{orabona2019modern}, the adversary here optimizes over an infinite-dimensional set of probability measures that changes in every round as its center $\hat \bP_t$ moves with each observation.
We play this game with the distributional best-response algorithm of \citep{chen2026oracle}, which is the best-response variant of the saddle-point scheme in \citep{orabona2019modern}. Our algorithm combines two existing methods: in each round, the adversary computes a near-optimal response to $x_t$ over $\bB^p_\rho(\hat \bP_t)$ using an efficient oracle, such as the budget allocation method of \citep{chen2026oracle}, and the learner takes a single projected subgradient step. Our contribution therefore lies not in the algorithm itself, but in establishing its convergence guarantees:

\begin{enumerate}[label=$\diamond$]
	\item We bound the suboptimality gap of our online algorithm relative to the offline optimum \eqref{eq:WDRO} in terms of the regret of the projected subgradient steps, the average Wasserstein distance between $\hat{\bP}_t$ and $\bP^\star$, and the oracle error (Theorem~\ref{thm::y_best_response}). The bound is \emph{instance-dependent}: unlike bounds based on uniform complexity measures in OCO, it depends only on the regularity of $\ell(x^\star,\cdot)$ at an optimal decision $x^\star$ and on the population distribution $\bP^\star$, without requiring uniform control over the entire decision set $\cX$ or the empirical measures $\hat{\bP}_t$. Moreover, it has no explicit dependence on the Wasserstein radius $\rho$.
	\item When $\bP^\star$ is light-tailed, measure concentration yields an explicit high-probability convergence rate for a fixed radius $\rho$ (Corollary~\ref{cor::convergence}).
\end{enumerate}

\subsection{Related Work}

\paragraph{Online Wasserstein DRO.}
An early treatment \citep{ido2021distributionally} studies online gradient and mirror-descent methods for Wasserstein DRO with streaming data.
The method of \citep{wang2025data} compresses the data stream by online clustering and uses approximate subgradient updates for losses that are piecewise affine in the uncertain parameter.
The certified training approach of \citep{sinha2018certifying} takes stochastic gradient steps over a fixed dataset for a Lagrangian penalty formulation of Wasserstein DRO, yielding a bound that is tight at a data-dependent radius.
The algorithm we analyze uses one distributional oracle query and one projected subgradient step per round, while keeping $\rho$ fixed.
Our analysis compares the averaged decision with the robust optimum centered at the population distribution $\bP^\star$.

\paragraph{Online Algorithms for Robust Optimization.}
Online learning methods solve robust optimization problems through repeated feasibility or separation steps \citep{ben2015oracle, ho2018online, ho2019exploiting}, or through perspective and Lagrangian reformulations \citep{postek2024first, tu2024max}.
For $f$-divergence ambiguity sets with discrete support, the adversary optimizes over a probability simplex, allowing primal-dual updates \citep{namkoong2016stochastic, aigner2023data} and online minimization of the dual \citep{qi2021online}.
In our setting, the adversary selects a distribution from a moving Wasserstein ball without assuming discrete support; its decision is therefore a probability measure rather than a vector in a simplex.

\subsection{Notation and Outline}

Let $\|\cdot\|$ denote the Euclidean norm, and let $[n]$ denote the set of positive integers up to $n \in \mathbb{N}$.
We write $\cP(\cZ)$ for the family of Borel probability measures on $\cZ \subseteq \RR^m$, equipped with the $p$-Wasserstein distance for $p \in [1, \infty)$.
We write $\EE_{\bP}[\ell(x, z)]$ for the expectation of $\ell(x, z)$ with respect to $z \sim \bP$.
We denote by $\Pi_\cX$ the Euclidean projection operator onto a closed and convex set $\cX$, and by $\partial_x f(x, \bQ)$ a subgradient of the convex function $f(\cdot, \bQ)$ at $x$.
The Lipschitz constant of a function $f: \cZ \rightarrow \RR$ is denoted by $\norm{f}_{\lip}$.
For $p \in (1,\infty)$ and a piecewise differentiable function $f(z) = \max_{k \in [K]} f_k(z)$, we write $\norm{f}_{\dot{H}^{1,p} (\bP)} := (\EE_{z \sim \bP}[\max_{k \in [K]} \norm{\nabla_z f_k(z)}^p])^{1/p}$.

The remainder of the paper is organized as follows. Section~\ref{sec:setup} introduces the problem setup and assumptions. In Section~\ref{sec:framework}, we analyze the convergence of the proposed algorithm, assuming access to an oracle for the (inner) worst-case expectation problem. Section~\ref{sec:numerics} presents numerical experiments on inventory decision-making under both synthetic and real demands to evaluate the performance and computational efficiency of the proposed algorithm. All proofs are relegated to the appendix.

\section{Problem Setup}
\label{sec:setup}

In this section, we formalize the structural assumptions required for our analysis. 

\begin{assumption}[Regularity]
	\label{asp::regularity}~
    \begin{enumerate}
        \item The data-generating distribution $\bP^\star$ has a finite $p$-th moment. 
        \item The feasible region $\cX \subseteq \mathbb{R}^n$ is nonempty, convex and compact, with diameter $D_\cX > 0$. The support of the random variable $\cZ \subseteq \mathbb{R}^m$ is nonempty, closed and convex.
        \item For any $x \in \cX$, there exists a constant $g > 0$ such that $|\ell(x, z)| \leq g (1 + \| z \|^p)$ for all $z \in \cZ$.
        \item  If $p = 1$, either $\cZ$ is compact or, for every $x \in \cX$, there exist $g > 0$ and $r \in (0, 1)$ such that $\ell(x, z) \le g(1 + \|z\|^r)$ for all $z \in \cZ$.
    \end{enumerate}
\end{assumption}
Assumption~\ref{asp::regularity}\,(i)--(iii) are standard in the literature and ensure that the optimization problem 
\begin{align}
	\label{eq:WDRO:t}
	\inf_{x \in \mathcal{X}} \; \sup_{\mathbb{Q} \in \bB^p_\rho(\hat \bP_t)} \big\{ f(x, \mathbb{Q}) := \mathbb{E}_{z \sim \mathbb{Q}}[\ell(x, z)] \big\}
\end{align}
has a finite optimal value \citep[Theorem~1]{gao2023distributionally}.
Assumption~\ref{asp::regularity}\,(iv) additionally ensures that the inner maximization admits an optimal distribution when $p=1$.

We next assume the following piecewise structure for the loss function.

\begin{assumption}[Piecewise Structure]
	\label{asp::convexity}
	The loss function $\ell(x, z) := \max_{k\in [K]} \ell_k(x, z)$, where for every $k\in [K]$, the function $\ell_k: \cX \times \cZ \to \bR$ is convex-concave and differentiable. Furthermore, $\ell(\cdot, z)$ is Lipschitz continuous with Lipschitz constant $G_{\cX} > 0$. 
\end{assumption}
For $p = 1$, this is the class of functions for which  \citep[Theorem~4.2]{mohajerin2018data} gives a finite convex reformulation of the worst-case expectation problem.
Under these conditions, the following lemma guarantees the existence of a saddle point and applies to both the empirical center $\hat{\bP}_t$ and the population center $\bP^\star$.

\begin{lemma}[Existence of Saddle Point]
	\label{lem::existence_saddle}
	Suppose Assumptions~\ref{asp::regularity} and~\ref{asp::convexity} hold, and let $\bP \in \cP(\cZ)$ have a finite $p$-th moment.
    Then,
	\[
	\min_{x \in \cX} \max_{\mathbb{Q} \in \bB^{p}_{\rho}({\mathbb{P}})} f(x, \mathbb{Q}) = \max_{\mathbb{Q} \in \bB^{p}_{\rho}({\mathbb{P}})} \min_{x \in \cX} f(x, \mathbb{Q}).
	\]
\end{lemma}

We also assume that the distribution $\mathbb{P}^\star$ is light-tailed.

\begin{assumption}[Light-tailed Distribution]
	\label{asp::light_tailed}
	The underlying data-generating distribution $\mathbb{P}^\star$ is light-tailed, that is, there exists $a > p \geq 1$ such that $\mathbb{E}_{z \sim \mathbb{P}^\star} [\exp(\|z\|^a)] < \infty$.
\end{assumption}

Our algorithm uses the following computational oracle to obtain a nearly worst-case distribution.

\begin{sloppypar}
	\begin{assumption}[Wasserstein Oracle]
		\label{asp::oracle}
		For every $\delta>0$, there exists a Wasserstein oracle $\mathcal{O}_{\W}(x,\bB^{p}_{\rho}(\bP))$ that returns a distribution $\bQ^\star \in \bB^{p}_{\rho}(\bP)$ satisfying
        \[
        f(x,\bQ^\star) \geq \max_{\bQ \in \bB^{p}_{\rho}(\bP)} f(x,\bQ)-\delta.
        \]
	\end{assumption}
\end{sloppypar}

Under the conditions of Lemma~\ref{lem::existence_saddle}, the budget allocation method in \citep[Algorithm~5]{chen2026oracle} satisfies Assumption~\ref{asp::oracle}, returning a $\delta$-optimal distribution with the guarantees established in \citep[Theorem~2]{chen2026oracle}.

Under these assumptions, Algorithm~\ref{alg::1} adapts the best-response scheme for saddle-point problems~\cite[\textsection~15]{orabona2019modern} by omitting the averaged dual iterate and re-centering the Wasserstein ball at $\hat{\bP}_t$ in each round, resulting in a time-varying dual feasible set.
The algorithm follows the structure of \citep[Algorithm~6]{chen2026oracle}.

\begin{figure*}[!t]
	\vspace{-1em}
	\begin{center}
		\begin{minipage}[t]{0.7\textwidth}
			\begin{algorithm}[H]
				\small
				\caption{\texttt{Online Distributional Best-Response}}
				\label{alg::1}
				\begin{algorithmic}
					\REQUIRE Initial decision $x_1 \in \cX$, step size $\eta_t > 0$ \\[0.5ex]
					\hspace{-1em}\textbf{Initialize:} $\bar{x}_1 = x_1$ \vspace{0.5ex}
					
					\FOR{$t=1,2,\dots, T$} \vspace{0.5ex}
					\STATE Query the oracle to find the nearly worst-case distribution: \vspace{-1.5ex}
					$$
					\mathbb{Q}_{t} \leftarrow \mathcal{O}_{\text{W}}(x_{t}, \bB^{p}_{\rho}(\hat{\mathbb{P}}_t))\;
					$$ \\ \vspace{-1.5ex}
					\STATE Update the decision via projected subgradient method: \vspace{-1.5ex}
					$$
					x_{t+1} = \Pi_{\cX}\left(x_{t} - \eta_t \, \partial_x f(x_{t}, \mathbb{Q}_t)\right)\;
					$$ \\ \vspace{-1.5ex}
					\STATE Maintain the average of the decisions: \vspace{-1.5ex}
					$$
					\bar{x}_{t+1} = \frac{1}{t+1} \sum_{i=1}^{t+1} x_i = \frac{t}{t+1}\bar{x}_{t} + \frac{1}{t+1} x_{t+1}\;
					$$ \\ \vspace{-1.5ex}
					\ENDFOR
					\vspace{0.5ex}
					
					\ENSURE $\bar x_T$
				\end{algorithmic}
			\end{algorithm}
		\end{minipage}
	\end{center}
\end{figure*}

\section{Convergence Analysis}
\label{sec:framework}

Before proceeding to the formal analysis, we define the learning objective within this online setting and establish the criterion for evaluating the performance of Algorithm~\ref{alg::1}. 
Our goal is to ensure that the robust estimate $\bar{x}_t$ produced by Algorithm~\ref{alg::1} converges to an optimal solution $x^\star\in \mathcal X$ of \eqref{eq:WDRO} according to a well-defined metric. We evaluate this convergence rate using the \emph{primal} suboptimality gap, a standard performance measure in minimax optimization when the focus is on the primal decision variable. Specifically, we seek to derive an upper bound for the quantity:
\[
\mathrm{Gap}(\bar{x}_T):=\max_{\mathbb{Q} \in \bB^{p}_{\rho}({\bP}^\star)} f(\bar{x}_T, \bQ) - \min_{x \in \mathcal{X}} \max_{\mathbb{Q} \in \bB^{p}_{\rho}({\bP}^\star)} f(x, \bQ).
\]
The following lemma plays a key role in controlling this suboptimality gap.
\begin{lemma}\label{lemma::decomposition}
	The suboptimality gap of the averaged iterate $\bar{x}_T$ admits the following upper bound:
	\[
	\begin{aligned}
		\mathrm{Gap}(\bar{x}_T)\leq& {\frac{1}{T}\sum_{t=1}^T \left( f(x_{t}, \bQ_t) - f(x^\star, \bQ_t) \right)}+ {\frac{1}{T}\sum_{t=1}^T \left( \max_{\bQ \in \bB^{p}_{\rho}({\bP}^\star)} f(x_{t}, \bQ) - \max_{\bQ \in \bB^{p}_{\rho}(\hat{\bP}_t)} f(x_{t}, \bQ) \right)}\\ 
		& + {\frac{1}{T}\sum_{t=1}^T \left(\max_{\bQ \in \bB^{p}_{\rho}(\hat{\bP}_t)} f(x^\star, \bQ) - \max_{\bQ \in \bB^{p}_{\rho}({\bP}^\star)} f(x^\star, \bQ) \right)} + {\delta}.
	\end{aligned}
	\]
\end{lemma}

According to this lemma, the suboptimality gap can be decomposed into four distinct components. The first component characterizes the regret of the \emph{learning step} in Algorithm~\ref{alg::1}, evaluated on the sequence of functions $\{f(\cdot, \bQ_t)\}_{t=1}^T$. The second and third components quantify the sensitivity of the worst-case expectations with respect to two Wasserstein balls, one centered at $\hat{\bP}_t$ and the other at $\bP^\star$. Finally, the fourth component accounts for the error incurred by the Wasserstein oracle.

Among these, the first component is the most straightforward to control, as it reduces to a standard regret analysis of the online projected subgradient method, which we present next.
\begin{lemma}
	\label{lem::online_gradient}
	Under Assumptions~\ref{asp::regularity} and~\ref{asp::convexity}, for any $T \geq 1$, the sequence $\{ (x_t, \bQ_t) \}_{t=1}^{T}$ generated by Algorithm~\ref{alg::1} with the stepsize $\eta_t = \frac{D_{\cX}}{G_{\cX} \sqrt{t}}$ satisfies
	\[
	\frac{1}{T}\sum_{t=1}^{T} \left(f(x_t, \bQ_t) -  f(x^\star, \bQ_t) \right) \leq \frac{3 G_{\cX} D_{\cX}}{2\sqrt{T}}.
	\]
\end{lemma}
Controlling the second and third components in Lemma~\ref{lemma::decomposition}, however, is more delicate. In particular, although $\hat{\bP}_t \to \bP^\star$, it is not immediate how this convergence translates into convergence of the corresponding worst-case expectations over Wasserstein balls centered at $\hat{\bP}_t$ and $\bP^\star$. The following proposition characterizes this effect.

\begin{proposition}
	\label{prop::convergence}
	Under Assumptions~\ref{asp::regularity} and~\ref{asp::convexity}, for any $x \in \cX$ and $t \ge 1$, we have
	\begin{align*}
		\sup_{\bQ \in \bB^p_\rho(\hat{\bP}_t)} f(x, \bQ)
		- \sup_{\bQ \in \bB^p_\rho(\bP^\star)} f(x, \bQ) \leq
		\begin{cases}
			\norm{\ell(x, \cdot)}_{\lip} \W_1(\hat{\bP}_t, \bP^\star),
			& \text{if } p = 1, \\
			\norm{\ell(x, \cdot)}_{\dot{H}^{1,q}(\bP^\star)}
			\W_p(\hat{\bP}_t, \bP^\star),
			& \text{if } p > 1,
		\end{cases}
	\end{align*}
	where $q > 1$ satisfies $1/p + 1/q = 1$ when $p > 1$.
	The same inequality holds after interchanging $\hat{\bP}_t$ and $\bP^\star$ throughout.
\end{proposition}

The coefficient in Proposition~\ref{prop::convergence} does not depend on the radius $\rho$,  which allows us to keep $\rho$ fixed as $T \to \infty$, unlike the vanishing radius required in~\citep{ido2021distributionally}.
Directly applying Proposition~\ref{prop::convergence} to the second term in Lemma~\ref{lemma::decomposition} at $x = x_t$, however, would introduce constants depending on the random and intermediate iterates $x_t$ (and on the empirical distribution $\hat{\bP}_t$ when $p > 1$).
In the following theorem, we refine the dual comparison for the second term so that the resulting bound depends only on the regularity of $\ell(x^\star, \cdot)$ at the optimal solution $x^\star$ and the population distribution $\bP^\star$.

\begin{theorem}
\label{thm::y_best_response}
    Suppose Assumptions~\ref{asp::regularity}, \ref{asp::convexity} and~\ref{asp::oracle} hold, and the observations are i.i.d. drawn from $\bP^\star$.
    Fix $\rho > 0$, $T \ge 1$, and $\tau \in (0, 1)$.
    Let $\bar{x}_T$ be generated by Algorithm~\ref{alg::1} with stepsize $\eta_t = D_\cX/(G_\cX \sqrt{t})$.
    The following guarantees hold with probability at least $1 - \tau$.
    
    \medskip
    \noindent
    \textbf{(i) Case $\bm{p = 1}$.}
    Suppose $\norm{\ell(x^\star, \cdot)}_{\lip} < \infty$. Then
    \begin{align*}
    \mathrm{Gap}(\bar{x}_T)
    &\le \frac{3G_\cX D_\cX}{2\sqrt{T}}
    + \frac{2\norm{\ell(x^\star, \cdot)}_{\lip}}{T}
    \sum_{t=1}^T \W_1(\hat{\bP}_t, \bP^\star)  + C_1 G_\cX D_\cX
    \sqrt{\frac{\log(2T/\tau)}{T}} + \delta,
    \end{align*}
    where $C_1>0$ depends only on $n$.
    
    \medskip
    \noindent
    \textbf{(ii) Case $\bm{p > 1}$.}
    Suppose Assumption~\ref{asp::light_tailed} additionally holds and $\norm{\ell(x^\star, \cdot)}_{\dot{H}^{1,q}(\bP^\star)} < \infty$, where $1/p + 1/q = 1$.
    Then
    \begin{align*}
    \mathrm{Gap}(\bar{x}_T)
    &\le \frac{3G_\cX D_\cX}{2\sqrt{T}}
    + \frac{C_2}{T}\sum_{t=1}^T \W_p(\hat{\bP}_t, \bP^\star)
    + C_3\sqrt{\frac{\log(2T/\tau)}{T}} + \delta,
    \end{align*}
    where $C_2, C_3>0$ depend only on $n$, $p$, $a$, $G_\cX D_\cX$, the growth constant $g$ for $\ell(x^\star,\cdot)$ in Assumption~\ref{asp::regularity}, the seminorm $\norm{\ell(x^\star,\cdot)}_{\dot{H}^{1,q}(\bP^\star)}$, and the exponential moment in Assumption~\ref{asp::light_tailed}.

\end{theorem}

The bounds above decompose the suboptimality gap of the averaged decision into four terms: the regret of the online projected subgradient steps, the average $p$-Wasserstein distance between the empirical measures $\hat{\bP}_t$ and $\bP^\star$, a concentration term of order $\sqrt{\log(2T/\tau)/T}$, and the Wasserstein oracle error.
Crucially, the coefficients of the Wasserstein terms require regularity of the loss in the uncertain parameter only at $x^\star$, rather than uniformly over $\cX$.

When $\bP^\star$ is light-tailed (Assumption~\ref{asp::light_tailed}), a high-probability bound on the average $p$-Wasserstein distance between $\hat{\bP}_t$ and $\bP^\star$ yields the following end-to-end convergence rate.

\begin{corollary}
\label{cor::convergence}
    Suppose the conditions of Theorem~\ref{thm::y_best_response} and Assumption~\ref{asp::light_tailed} hold, with $p \geq 1$ and $p \neq m/2$.
    There exist constants $C_1 \geq 1$ and $C_2 > 0$ such that, for any $\tau \in (0,1)$ and $T \geq \max\{3,C_1\log(T/\tau)\}$, choosing the oracle tolerance
    \begin{align*}
    \delta
    = \left(\frac{\log(T/\tau)}{T}\right)^{\min\{p/m,\,1/2\}}
    \end{align*}
    yields, with probability at least $1-\tau$,
    \begin{align*}
    \mathrm{Gap}(\bar{x}_T)
    \leq C_2
    \left(\frac{\log(T/\tau)}{T}\right)^{\min\{1/m,\,1/(2p)\}}.
    \end{align*}
    The constants depend only on $n$, $m$, $p$, $a$, $G_\cX D_\cX$, the exponential moment in Assumption~\ref{asp::light_tailed}, and the following quantities:
    \begin{itemize}
        \item If $p=1$, the Lipschitz seminorm $\norm{\ell(x^\star,\cdot)}_{\lip}$.
        \item If $p>1$, the growth constant $g$ for $\ell(x^\star,\cdot)$ in Assumption~\ref{asp::regularity} and the seminorm $\norm{\ell(x^\star,\cdot)}_{\dot{H}^{1,q}(\bP^\star)}$, where $1/p+1/q=1$.
    \end{itemize}
\end{corollary}

The established bound highlights the interplay between the dimensionality of the uncertainty set and the convergence rate of the proposed online algorithm. In particular, it confirms that despite operating on streaming data, the online framework preserves the statistical guarantees dictated by measure concentration, consistent with the offline DRO results of \citep{mohajerin2018data}.

\section{Numerical Experiments}
\label{sec:numerics}

In this section, we compare Online DRO and Online ERM in an inventory problem with synthetic demand and daily bike-rental demand. In our setting, a decision-maker chooses
an order quantity $x\in[0,1]$ before observing a demand index
$z\in[0,1]$. Effective demand is $z^2$, procurement costs are $x+x^2$,
and each unit sold yields revenue $4$. The loss is the negative profit:
\[
\begin{aligned}
\ell(x,z)
=x+x^2-4\min\{x,z^2\} 
=\max\{x^2-3x,\,x+x^2-4z^2\}.
\end{aligned}
\]
Online ERM updates the order quantity based on the average loss over past demand observations, whereas Online DRO uses the worst-case expected loss over an ambiguity set of demand distributions to account for uncertainty in the observed data.

For Online DRO, we use Algorithm~\ref{alg::1} with $p=2$, $\rho=0.01$, and initial decision $x_1=0.5$.
Online ERM uses the same initialization and performs a projected subgradient update upon receiving each new observation. Within each
experiment, both methods receive the same observations and use the
same step sizes. We evaluate their averaged decisions
$\bar x_t=t^{-1}\sum_{s=1}^t x_s$, for $t \in [T]$.

\subsection{Synthetic Demand}
\label{sec:numerics:synthetic}

We first compare decision quality and computation time under a synthetic
discrete demand distribution. The demand index $z$ takes values
$(0.25,\allowbreak 0.45,\allowbreak 0.65,\allowbreak 0.8,\allowbreak 0.95)$
with respective probabilities
$(0.1,\allowbreak 0.25,\allowbreak 0.35,\allowbreak 0.25,\allowbreak 0.05)$.
We generate 200 independent streams, each consisting of $T=10^4$
i.i.d.\ observations, and use $\eta_t=1/(3\sqrt t)$ for both online
methods. At each reported iteration $10\le t\le10^4$, we also compute
Offline DRO and Offline ERM by minimizing the robust and nominal
empirical objectives, respectively, over the first $t$ observations.

Figure~\ref{fig:numerics} (left) reports the primal suboptimality gap
defined in Section~\ref{sec:framework}, with the underlying demand
distribution as $\mathbb P^\star$. All four methods are therefore
evaluated against the same population robust optimum. The mean gap of Online DRO decreases from $0.026$ at $t=10$ to $5.54\times10^{-6}$ at $t=10^4$ and is only slightly larger than that of its offline counterpart. In contrast, both online and offline ERM exhibit nonvanishing suboptimality gaps, reflecting the mismatch between the nominal objective they optimize and the robust objective used for evaluation.

Figure~\ref{fig:numerics} (right) compares the cumulative runtimes through iteration $t$. For each online method, this is the time required to process the first $t$ observations sequentially in a single pass. For each offline method, it is the sum of the runtimes of all batch fits performed up to iteration $t$. At $t=10^4$, the median times for
Online DRO, Online ERM, Offline DRO, and Offline ERM are
$(0.0225,\allowbreak 0.00368,\allowbreak 0.5865,\allowbreak 0.00774)$
seconds, respectively. Even though each offline fit is inexpensive in this
scalar problem with at most five distinct demand values, both online methods achieve lower cumulative runtimes than their respective offline counterparts by using sequential updates that avoid repeated batch optimization.

\begin{figure}[!ht]
\centering
\includegraphics[width=\linewidth]{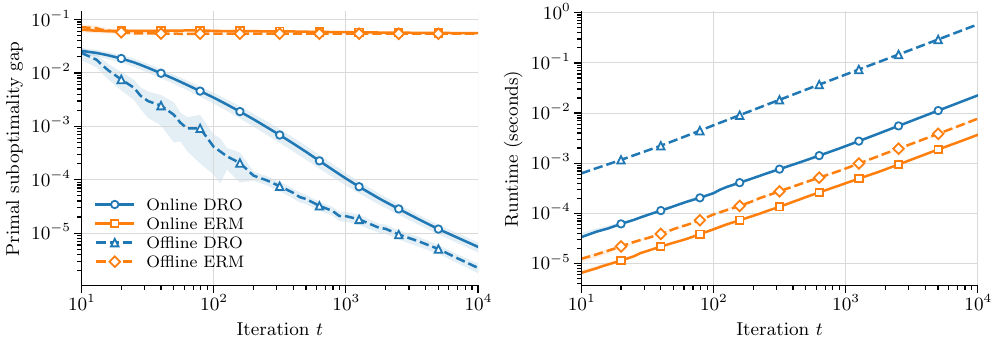}
\caption{Performance under synthetic demand. Left: mean primal suboptimality gap over 200 independent streams, with bands $\pm1.96\,\mathrm{SE}$, where
$\mathrm{SE}=\mathrm{SD}/\sqrt{200}$ and $\mathrm{SD}$ is the sample
standard deviation. Right: median runtime over five repeated runs
on one fixed stream, with interquartile bands. Both panels use
logarithmic axes.}
\label{fig:numerics}
\end{figure}
\FloatBarrier

\subsection{Daily Bike-Rental Demand}
\label{sec:numerics:real}

We next compare mean and worst-case expected profits using 2012 daily
rental counts from the UCI Bike Sharing dataset. For a total daily
count $C$, we set the demand index to $z=C/10000$. The online methods
process 274 daily observations from January through September in
calendar order, with $\eta_t=1/\sqrt t$. The empirical distribution
of the following 92 days, from October through December, defines
the fixed reference $\mathbb P^\star$.

At each iteration, we evaluate both averaged decisions by their mean
profit under $\mathbb P^\star$ and their worst-case expected profit.
The latter is the lowest expected profit over the Wasserstein
ambiguity set centered at $\mathbb P^\star$, with $\rho=0.01$ for
both methods.

Figure~\ref{fig:numerics:bike} shows the tradeoff between mean and
worst-case expected profits. Online DRO has higher worst-case
expected profit from $t=13$ onward. Its largest shortfall in mean
profit relative to Online ERM is $18.4\%$. At $t=274$, the mean
profits of Online DRO and Online ERM are $0.485$ and $0.497$,
respectively. Their worst-case expected profits are $0.309$ and
$0.194$, corresponding to a $59.3\%$ increase for Online DRO. Overall, Online DRO achieves substantially higher worst-case expected profit with only a small reduction in mean profit relative to Online ERM, illustrating the value of accounting for demand uncertainty.

\begin{figure}[!ht]
\centering
\includegraphics[width=\linewidth]{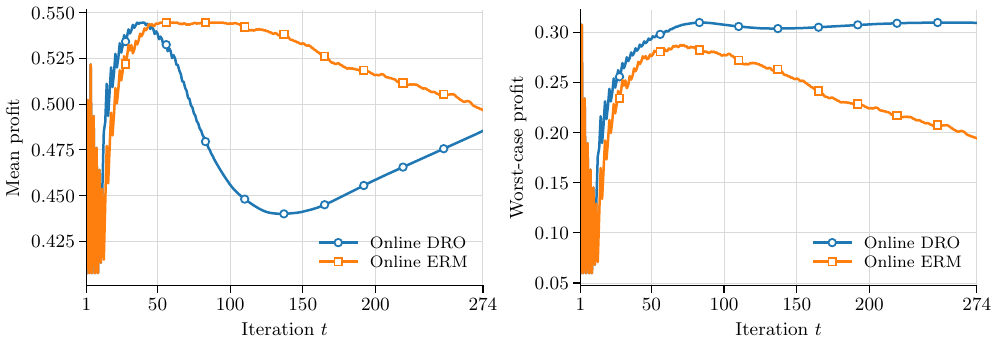}
\caption{Daily bike-rental demand. Left: mean profit under $\bP^\star$. Right: worst-case expected profit over its Wasserstein ambiguity set. Higher values are better in both panels.}
\label{fig:numerics:bike}
\end{figure}

\section*{Acknowledgments}
Salar Fattahi is supported, in part, by the NSF CAREER grant CCF-2337776 and ONR grant N00014-26-1-2074. Soroosh Shafiee is supported, in part, by the NSF CAREER grant ECCS-2541066.

\section*{Declaration of generative AI use}

The authors used ChatGPT to assist with language editing,
reviewing mathematical arguments, and checking numerical calculations.
The authors reviewed and revised the resulting material and assume
full responsibility for the content and conclusions of this paper.

\bibliography{ORL}

\newpage
\appendix
\section{Proofs}

\subsection{Proof of Lemma~\ref{lem::existence_saddle}}
The claim follows from Theorem~1(iii) and Lemma~2 in \citep{shafiee2025nash} with reference distribution $\bP$ and transportation cost $c(z, \hat z) = \|z - \hat z\|^p$.
Assumptions~\ref{asp::regularity} and~\ref{asp::convexity} give the required continuity, moment and integrability conditions.
The uniform Lipschitz bound in $x$ gives local uniform integrability, and compactness of $\cX$ gives inf-compactness with $\bQ_0 = \bP$.
For $p > 1$, each differentiable concave loss piece admits an affine upper bound, so $\ell(x, \cdot)$ has at most linear upper growth.
For $p = 1$, compactness of $\cZ$ or the assumed sublinear upper bound gives the required growth condition.
This completes the proof.
\hfill \qed

\subsection{Proof of Lemma~\ref{lemma::decomposition}}
Using Jensen's inequality, we can write
\begin{align*}
	\mathrm{Gap}(\bar x_T)&=\max_{\mathbb{Q} \in \bB^{p}_{\rho}({\bP}^\star)} f(\bar{x}_T, \bQ) - \min_{x \in \mathcal{X}} \max_{\mathbb{Q} \in \bB^{p}_{\rho}({\bP}^\star)} f(x, \bQ) \\
	&\leq  \frac{1}{T} \sum_{t=1}^T \left( \max_{\bQ \in \bB^{p}_{\rho}({\bP}^\star)} f(x_t, \bQ) - \max_{\bQ \in \bB^{p}_{\rho}({\bP}^\star)} f(x^\star, \bQ) \right),
\end{align*}
where $x^\star\in \argmin_{x\in \cX}\max_{\mathbb{Q} \in \bB^{p}_{\rho}({\bP}^\star)} f(x, \bQ)$. Next, we have
\begin{align*}
	\max_{\bQ \in \bB^{p}_{\rho}({\bP}^\star)} f(x_t, \bQ) - \max_{\bQ \in \bB^{p}_{\rho}({\bP}^\star)} f(x^\star, \bQ)\leq& { f(x_{t}, \bQ_t) - f(x^\star, \bQ_t) }\\
	&+ {\max_{\bQ \in \bB^{p}_{\rho}({\bP}^\star)} f(x_{t}, \bQ) - \max_{\bQ \in \bB^{p}_{\rho}(\hat{\bP}_t)} f(x_{t}, \bQ) }\\
	&+ { \max_{\bQ \in \bB^{p}_{\rho}(\hat{\bP}_t)} f(x^\star, \bQ) - \max_{\bQ \in \bB^{p}_{\rho}({\bP}^\star)} f(x^\star, \bQ) }\\
	&+ {\delta}.
\end{align*}
The above inequality is obtained by noting that $\max_{\bQ \in \bB^{p}_{\rho}(\hat{\bP}_t)} \{f(x^\star, \bQ)\}- f(x^\star, \bQ_t)\geq 0$ and applying the Wasserstein oracle from Assumption~\ref{asp::oracle}, which ensures $f(x_{t}, \bQ_t) - \max_{\bQ \in \bB^{p}_{\rho}(\hat{\bP}_t)} \{f(x_{t}, \bQ)\} +  \delta\geq 0$. Combining the above two inequalities completes the proof.
\hfill \qed

\subsection{Proof of Lemma~\ref{lem::online_gradient}}
Write $g_t = \partial_x f(x_t, \bQ_t)$ for the subgradient used in Algorithm~\ref{alg::1}. Each function $f(\cdot,\bQ_t)$ is convex and $G_\cX$-Lipschitz on $\cX$. Therefore, $\|g_t\| \le G_\cX$.
Each $f(\cdot, \bQ_t)$ is convex on $\cX$ and the step sizes $\eta_t = D_\cX / (G_\cX \sqrt{t})$ are nonincreasing, so the regret bound of projected online subgradient descent in \citep[Theorem~2.13]{orabona2019modern} gives
\begin{align*}
	\sum_{t = 1}^T \big( f(x_t, \bQ_t) - f(x^\star, \bQ_t) \big) & \le \frac{D_\cX^2}{2 \eta_T} + \frac{1}{2} \sum_{t = 1}^T \eta_t \|g_t\|^2 \\
	& \le \frac{G_\cX D_\cX \sqrt{T}}{2} + \frac{G_\cX D_\cX}{2} \sum_{t = 1}^T \frac{1}{\sqrt{t}} \le \frac{3 G_\cX D_\cX \sqrt{T}}{2},
\end{align*}
where the last inequality uses $\sum_{t = 1}^T t^{-1/2} \le 2 \sqrt{T}$.
Dividing by $T$ completes the proof.
\hfill \qed

\subsection{Proof of Proposition~\ref{prop::convergence}}
For ease of notation, throughout the proof of Proposition~\ref{prop::convergence}, we suppress the dependence on the decision variable $x$ and simply write $\ell(z)$ in lieu of $\ell(x,z)$. We now state and prove a slight generalization of Proposition~\ref{prop::convergence}.

\begin{lemma}
	\label{lem-general}
	Under Assumptions~\ref{asp::regularity}, \ref{asp::convexity}, given a pair of distributions $\bP_1, \bP_2 \in \cP(\cZ)$ with finite $p$-th moments, we have:
	\[
	\sup_{\QQ \in \bB^{p}_{\rho}(\bP_1)} \EE_{z \sim \bQ} [\ell(z)] - \sup_{\QQ \in \bB^{p}_{\rho}(\bP_2)} \EE_{z \sim \bQ} [\ell(z)]
	\leq
	\left\{
	\begin{aligned}
		& \| \ell(\cdot) \|_\lip \W_1(\bP_1, \bP_2), && \mathrm{if}\ p = 1,\\
		& \norm{\ell(\cdot)}_{\dot{H}^{1,q} (\bP_2)} \W_p(\bP_1, \bP_2), && \mathrm{if}\
		p > 1,\\
	\end{aligned}
	\right.
	\]
	where $q > 1$ satisfies $\frac{1}{p} + \frac{1}{q} = 1$ when $p > 1$.
\end{lemma}

To prove Lemma~\ref{lem-general}, we need the next two helper lemmas.

\begin{lemma}
	\label{lem::aux_concave}
	Under Assumptions~\ref{asp::regularity} and~\ref{asp::convexity}, for $k\in [K]$, define the function $h_k: \RR^m \times \RR_+ \rightarrow \RR \cup \{+\infty\}$ such that
	\[
	h_k(z, \lambda) := \sup_{z' \in \cZ} \ell_k(z') - \lambda \norm{z' - z}^p,\ \forall \lambda \geq 0,\, z \in \RR^m.
	\]
	For any fixed $\lambda \geq 0$, the function $h_k(z, \lambda)$ is concave in $z \in \RR^m$. Further, define $h(z, \lambda) := \max_{k\in [K]} h_k(z, \lambda)$. We have
	\[
	h(z, \lambda) = \sup_{z' \in \cZ} \ell(z') - \lambda \norm{z' - z}^p,\ \forall \lambda \geq 0,\, z \in \RR^m.
	\]
\end{lemma}
\begin{proof}
	For any fixed $\lambda \geq 0$, the function $\ell_k(z') - \lambda \norm{z' - z}^p$ is jointly concave in $(z', z) \in \cZ \times \RR^m$ under Assumption~\ref{asp::convexity}. Therefore, $\sup_{z' \in \cZ} \ell_k(z') - \lambda \norm{z' - z}^p$ is concave in $z \in \RR^m$, thereby proving the first statement. To prove the second statement, we write
	\[
	\begin{aligned}
		h(z, \lambda) &:= \max_{k\in [K]} h_k(z, \lambda)\\
		&= \max_{k\in [K]} \sup_{z' \in \cZ} \ell_k(z') - \lambda \norm{z' - z}^p\\
		&= \sup_{z' \in \cZ} \max_{k\in [K]} \ell_k(z') - \lambda \norm{z' - z}^p\\
		&= \sup_{z' \in \cZ} \ell(z') - \lambda \norm{z' - z}^p.
	\end{aligned}
	\]
	This concludes the proof.
\end{proof}

\begin{lemma}
	\label{lem::p>1}
	Under the conditions of Lemma~\ref{lem::aux_concave}, assume that $p > 1$ and $\lambda > 0$. Then $h_k(\cdot, \lambda)$ is finite and differentiable on $\RR^m$, and $\norm{ \nabla_z h_k (z, \lambda)} \leq \norm{ \nabla \ell_k(z) }$ for every $z \in \cZ$.
\end{lemma}

\begin{proof}
	For any fixed $\lambda > 0$, the function $\ell_k(z') - \lambda \norm{z' - z}^p$ is strictly concave in $z'$ due to $p > 1$ and the concavity of $\ell_k$. Therefore, since $\cZ \subseteq \bR^m$ is convex and closed, and
	\[
	\limsup_{\norm{z'}\rightarrow \infty} \frac{\ell_k(z')}{\norm{z'-z}^p} \leq 0,
	\]
	we obtain that $-(\ell_k(z') - \lambda \norm{z' - z}^p)$ is coercive with respect to $z'$, which in turn implies that $\argmax_{z'\in \cZ} \ell_k(z') - \lambda \norm{z' - z}^p$ is nonempty and has a unique element.
	Denote $z^\star = \argmax_{z' \in \cZ} \ell_k(z') - \lambda \norm{z' - z}^p$.
	Since $\ell_k(z') - \lambda \norm{z' - z}^p$ is continuously differentiable with respect to $z$, the maximization can be restricted to a compact subset of $\cZ$ containing $z^\star$. By Danskin's theorem, $h_k(\cdot, \lambda)$ is differentiable at $z$, and
	\begin{align}
		\label{eq::nabla-h_k}
		\begin{aligned}
			\nabla_z h_k(z, \lambda) &= \frac{\partial}{\partial z} \left[ \ell_k(z') - \lambda \norm{z' - z}^p\right] \bigg|_{z' = z^\star}
			= \lambda p \norm{z^\star - z}^{p-2} (z^\star - z),
		\end{aligned}
	\end{align}
	where, with a slight abuse of notation, the right-hand side is read as $0$ when $z^\star = z$.
	Since $\cZ$ is convex and closed, by the optimality of $z^\star$, we have
	\[
	\langle \nabla \ell_k(z^\star) - \lambda p \norm{z^\star - z}^{p-2} (z^\star - z), z'' - z^\star\rangle \leq 0,\ \forall z'' \in \cZ,
	\]
	which implies that
	\[
	\langle \nabla \ell_k(z^\star) - \nabla_z h_k(z, \lambda), z'' - z^\star\rangle \leq 0,\ \forall z'' \in \cZ.
	\]
	Let $z'' = z \in \cZ$, and note that, according to \eqref{eq::nabla-h_k}, $- \nabla_z h_k(z, \lambda)$ is in the same direction as $z - z^\star$. Therefore, we have
	\[
	\langle \nabla \ell_k(z^\star) - \nabla_z h_k(z, \lambda), -\nabla_z h_k(z, \lambda)\rangle \leq 0,
	\]
	which leads to
	\[
	\langle \nabla \ell_k(z^\star), \nabla_z h_k(z, \lambda)\rangle \geq \norm{\nabla_z h_k(z, \lambda)}^2.
	\]
	Since $\ell_k$ is concave, we have
	\[
	\langle \nabla \ell_k(z^\star) - \nabla \ell_k(z), z^\star - z \rangle \leq 0.
	\]
	Again, using \eqref{eq::nabla-h_k}, we have
	\[
	\langle \nabla \ell_k(z^\star) - \nabla \ell_k(z), \nabla_z h_k(z, \lambda) \rangle \leq 0.
	\]
	It follows that
	\begin{align*}
		\norm{\nabla_z h_k(z, \lambda)}^2 \leq \langle \nabla \ell_k(z^\star), \nabla_z h_k(z, \lambda)\rangle
		\leq  \langle \nabla \ell_k(z), \nabla_z h_k(z, \lambda)\rangle \leq \norm{\nabla \ell_k(z)} \norm{\nabla_z h_k(z, \lambda)}.
	\end{align*}
	Indeed, if $\nabla_z h_k(z, \lambda) = 0$, the statement of the lemma holds trivially. Moreover, if $\nabla_z h_k(z, \lambda) \neq 0$, the above inequality leads to
	\[
	\norm{\nabla_z h_k(z, \lambda)} \leq \norm{\nabla \ell_k(z)}.
	\]
	This completes the proof.
\end{proof}

\begin{proof}[Proof of Lemma~\ref{lem-general}]
	Define the value function
	\[ V_\rho(\mathbb{P}) = \sup_{\QQ \in \bB^{p}_{\rho}(\bP)} \EE_{z \sim \bQ} [\ell(z)]. \]
	By \citep[Theorem~1]{blanchet2019quantifying}, which applies to every $\bP$ with a finite $p$-th moment because $\ell$ is continuous and, by Assumption~\ref{asp::regularity}, integrable under $\bP$, we can reformulate the value function as
	\[ V_\rho(\mathbb{P}) = \inf_{\lambda \geq 0} \{ \lambda \rho + \mathbb{E}_{\mathbb{P}}[h(z, \lambda)] \}, \ \text{where} \ h(z, \lambda) := \sup_{z' \in \cZ} \{ \ell(z') - \lambda \|z' - z\|^p \}. \]
	The growth condition in Assumption~\ref{asp::regularity} also shows that $V_\rho(\mathbb{P}_2)$ is finite.
	For any $\epsilon > 0$, by the definition of the infimum, there exists a $\lambda_\epsilon \geq 0$ such that
	\[
	\lambda_\epsilon \rho + \mathbb{E}_{\mathbb{P}_2}[h(z, \lambda_\epsilon)] \leq V_\rho(\mathbb{P}_2) + \epsilon.
	\]
	By the suboptimality of $\lambda_\epsilon$ in the minimization problem for $V_\rho(\mathbb{P}_1)$, we have
	\[ V_\rho(\mathbb{P}_1) \leq \lambda_\epsilon \rho + \mathbb{E}_{\mathbb{P}_1}[h(z, \lambda_\epsilon)]. \]
	Subtracting these inequalities yields
	\begin{align}
		\label{eq::diff_bound}
		V_\rho(\mathbb{P}_1) - V_\rho(\mathbb{P}_2)
		\leq \mathbb{E}_{\mathbb{P}_1}[h(z, \lambda_\epsilon)] - \mathbb{E}_{\mathbb{P}_2}[h(z, \lambda_\epsilon)] + \epsilon.
	\end{align}
	Both expectations are well defined, because $h(z, \lambda_\epsilon) \ge \ell(z)$ by the choice $z' = z$, the loss $\ell$ is integrable under $\bP_1$ and $\bP_2$, and $\mathbb{E}_{\mathbb{P}_2}[h(z, \lambda_\epsilon)]$ is finite by the choice of $\lambda_\epsilon$.
	Letting $\pi$ be the optimal coupling for $\W_p(\mathbb{P}_1, \mathbb{P}_2)$, which exists because both distributions have finite $p$-th moments, we can rewrite the above expectation as
	\[ \mathbb{E}_{\mathbb{P}_1}[h(z, \lambda_\epsilon)] - \mathbb{E}_{\mathbb{P}_2}[h(z, \lambda_\epsilon)] = \mathbb{E}_{(z_1, z_2) \sim \pi} [h(z_1, \lambda_\epsilon) - h(z_2, \lambda_\epsilon)]. \]
	
	\paragraph{Case 1 ($\bm{p=1}$):}
	In this case, the function $h(z, \lambda) = \sup_{z' \in \cZ} \{ \ell(z') - \lambda \|z' - z\| \}$.
	First, by the properties of the Pasch-Hausdorff envelope in \citep[Example~9.11]{rockafellar1998variational}, the function $h(\cdot, \lambda)$ is, for any $\lambda \geq 0$, either identically $+\infty$ or finite and $\lambda$-Lipschitz continuous.
	Second, we argue that we can restrict the search for $\lambda$ to the interval $[0, \| \ell \|_\lip]$. Suppose $\lambda \ge \| \ell \|_\lip$. Since $\ell$ is $\| \ell \|_\lip$-Lipschitz, for any $z' \in \cZ$ we have $\ell(z') \leq \ell(z) + \| \ell \|_\lip \cdot \|z'-z\|$. Thus, we may conclude that
	\[
	\ell(z') - \lambda \|z' - z\| \leq \ell(z) + (\| \ell \|_\lip - \lambda) \|z' - z\| \leq \ell(z),
	\]
	where the last inequality holds because $\| \ell \|_\lip - \lambda \le 0$. Since the value $\ell(z)$ is attained at $z' = z$, it follows that $h(z, \lambda) = \ell(z)$ for all $\lambda \geq \| \ell \|_\lip$. In this regime, the dual objective $\lambda \rho + \mathbb{E}[h(z, \lambda)]$ is strictly increasing in $\lambda$. Therefore, the infimum must be attained at some $\lambda \in [0, \| \ell \|_\lip]$, and we can assume $\lambda_\epsilon \leq \| \ell \|_\lip$ without loss of optimality.
	Finally, $h(\cdot, \lambda_\epsilon)$ is finite because $\mathbb{E}_{\mathbb{P}_2}[h(z, \lambda_\epsilon)]$ is finite, so it is $\lambda_\epsilon$-Lipschitz continuous.
	Using this property and the optimal coupling $\pi \in \Gamma(\mathbb{P}_1, \mathbb{P}_2)$, we have
	\begin{align*}
		\mathbb{E}_{\mathbb{P}_1}[h(z, \lambda_\epsilon)] - \mathbb{E}_{\mathbb{P}_2}[h(z, \lambda_\epsilon)] &= \mathbb{E}_{(z_1, z_2) \sim \pi} [h(z_1, \lambda_\epsilon) - h(z_2, \lambda_\epsilon)] \\
		&\leq \mathbb{E}_\pi [\lambda_\epsilon \|z_1 - z_2\|] \\
		&\leq \| \ell \|_\lip \cdot \W_1(\mathbb{P}_1, \mathbb{P}_2).
	\end{align*}
	Substituting this into \eqref{eq::diff_bound} and taking $\epsilon \downarrow 0$ completes the proof for $p=1$.
	
	\paragraph{Case 2 ($\bm{p>1}$):}
	First, we look at the case where $\lambda_{\epsilon} = 0$. It follows that $h(z, 0) = \sup_{z\in \cZ} \ell(z)$, which is a finite constant because $\mathbb{E}_{\mathbb{P}_2}[h(z, 0)] \le V_\rho(\mathbb{P}_2) + \epsilon$. Substituting this into \eqref{eq::diff_bound} and taking $\epsilon \downarrow 0$, we have
	\[
	V_\rho(\mathbb{P}_1) - V_\rho(\mathbb{P}_2) \leq 0,
	\]
	which completes the proof.
	Next, assume that $\lambda_{\epsilon} > 0$. One can write
	\begin{equation}
		\label{eq::p>1}
		\begin{aligned}
			\EE_{\bP_1}[h(z, \lambda_{\epsilon})] - \EE_{\bP_2}[h(z, \lambda_{\epsilon})]
			&= \EE_{(z_1,z_2) \sim \pi} \left[h(z_1, \lambda_\epsilon) - h(z_2, \lambda_\epsilon) \right]\\
			&= \EE_{(z_1,z_2) \sim \pi} \left[\max_{k\in [K]} h_k(z_1, \lambda_\epsilon) - \max_{k\in [K]} h_k(z_2, \lambda_\epsilon) \right]\\
			& \leq \EE_{(z_1,z_2) \sim \pi} \left[ \max_{k \in [K]} \big( h_k(z_1, \lambda_\epsilon) - h_k(z_2, \lambda_\epsilon) \big) \right]\\
			&\stackrel{(a)}{\leq} \EE_{(z_1,z_2) \sim \pi} \left[ \max_{k \in [K]} \nabla_{z} h_k(z_2, \lambda_\epsilon)^\top (z_1 - z_2)\right]\\
			&\stackrel{(b)}{\leq} \EE_{(z_1,z_2) \sim \pi} \left[ \max_{k \in [K]} \norm{\nabla \ell_k(z_2)} \norm{z_1 - z_2} \right]\\
			&\stackrel{(c)}{\leq} \left( \EE_{z \sim \bP_2} \left[ \max_{k \in [K]} \norm{ \nabla \ell_k(z)}^q \right] \right)^{1/q} \left( \EE_{(z_1,z_2) \sim \pi} \norm{z_1 - z_2}^p \right)^{1/p} \\
			&= \norm{\ell(\cdot)}_{\dot{H}^{1,q} (\bP_2)} \W_p(\bP_1, \bP_2).
		\end{aligned}
	\end{equation}
	Here, the first inequality holds because $\max_{k \in [K]} a_k - \max_{k \in [K]} b_k \le \max_{k \in [K]} (a_k - b_k)$ for any real numbers $a_1, \dots, a_K$ and $b_1, \dots, b_K$.
	Moreover, \((a)\) follows from the differentiability and concavity of
	\(h_k(\cdot, \lambda_\epsilon)\), as established in
	Lemma~\ref{lem::p>1} and Lemma~\ref{lem::aux_concave}, respectively.
	Next, $(b)$ follows from the Cauchy-Schwarz inequality and Lemma~\ref{lem::p>1}.
	Finally, $(c)$ follows from H\"{o}lder's inequality, and the last equality uses the definition of $\norm{\ell(\cdot)}_{\dot{H}^{1,q} (\bP_2)}$ for the piecewise loss and the optimality of $\pi$.
	Combining \eqref{eq::diff_bound} and \eqref{eq::p>1} yields
	\[
	V_\rho(\mathbb{P}_1) - V_\rho(\mathbb{P}_2) \leq \norm{\ell(\cdot)}_{\dot{H}^{1,q} (\bP_2)} \W_p(\bP_1, \bP_2) + \epsilon.
	\]
	Letting $\epsilon \rightarrow0$ completes the proof of Lemma~\ref{lem-general}.
	The empirical distribution $\hat{\bP}_t$ has finite support, and $\bP^\star$ has a finite $p$-th moment by Assumption~\ref{asp::regularity}.
	Applying Lemma~\ref{lem-general} with $(\bP_1, \bP_2) = (\hat{\bP}_t, \bP^\star)$ gives the displayed bound in Proposition~\ref{prop::convergence}.
	Interchanging the two centers gives its second assertion.
\end{proof}

\subsection{Proof of Theorem~\ref{thm::y_best_response}}
For $x \in \cX$, $z \in \cZ$, and $\lambda \ge 0$, define
\begin{align*}
    h(x, z, \lambda)
    := \sup_{z' \in \cZ} \left\{\ell(x, z') - \lambda\norm{z' - z}^p\right\}.
\end{align*}
We first establish a uniform concentration bound for $h(x, z, \lambda)$ when $\lambda$ is large.

\begin{lemma}
    \label{lem::large_lambda}
    Suppose Assumptions~\ref{asp::regularity} and~\ref{asp::convexity} hold, and the observations are i.i.d. drawn from~$\bP^\star$.
    Fix $T \ge 1$ and $\tau \in (0, 1)$.
    \begin{enumerate}
        \item[(i)] \textbf{Case $\bm{p = 1}$.}
        Suppose $\norm{\ell(x^\star, \cdot)}_{\lip} < \infty$.
        Then, with probability at least $1 - \tau$, simultaneously for all $t \le T$, $x \in \cX$, and $\lambda \ge \norm{\ell(x^\star, \cdot)}_{\lip}$, we have
        \begin{align*}
            \EE_{\bP^\star}[h(x, z, \lambda) - \ell(x^\star, z)]
            - \EE_{\hat{\bP}_t}[h(x, z, \lambda) - \ell(x^\star, z)]
            \le C G_\cX D_\cX \sqrt{\frac{\log(2T/\tau)}{t}},
        \end{align*}
        where $C$ depends only on $n$.
        \item[(ii)] \textbf{Case $\bm{p > 1}$.}
        Suppose Assumption~\ref{asp::light_tailed} additionally holds, and let $g^\star$ be the growth constant in Assumption~\ref{asp::regularity} at $x^\star$.
        Then, with probability at least $1 - \tau$, simultaneously for all $t \le T$, $x \in \cX$, and $\lambda \ge 2^p g^\star$, we have
        \begin{align*}
            \EE_{\bP^\star}[h(x, z, \lambda)] - \EE_{\hat{\bP}_t}[h(x, z, \lambda)]
            \le C\left(\sqrt{\frac{\log(2T/\tau)}{t}} + \frac{\log(2T/\tau)}{t}\right),
        \end{align*}
        where $C$ depends only on $n$, $p$, $a$, $G_\cX D_\cX$, $g^\star$ and the exponential moment in Assumption~\ref{asp::light_tailed}.
    \end{enumerate}
\end{lemma}

\begin{proof}
    Define
    \begin{align*}
        \bar \lambda 
        := \begin{cases}
            \norm{\ell(x^\star, \cdot)}_{\lip}, & \text{if } p = 1 \\
            2^p g^\star, & \text{if } p > 1
        \end{cases} \quad \text{and} \quad
        \tilde h(x, z, \lambda)
        := \begin{cases}
            h(x, z, \lambda) - \ell(x^\star, z), & \text{if } p = 1 \\
            h(x, z, \lambda), & \text{if } p > 1,
        \end{cases}
    \end{align*}
    and denote the empirical error by
    \begin{align*}
    \Delta_t(x, \lambda)
    := \EE_{\bP^\star}[\tilde h(x, z, \lambda)] - \EE_{\hat{\bP}_t}[\tilde h(x, z, \lambda)].
    \end{align*}
    We proceed in four steps.
    
    \paragraph{Step 1 (regularity):}
    For any $x_1, x_2 \in \cX$, $z \in \cZ$, and $\lambda \ge 0$, since $\ell(\cdot, z')$ is $G_\cX$-Lipschitz for every $z' \in \cZ$, taking the supremum over $z' \in \cZ$ yields
    \begin{align}
    \label{eq::h_tilde_lip_x}
    \begin{aligned}
        |\tilde h(x_1, z, \lambda) - \tilde h(x_2, z, \lambda)|
        &= |h(x_1, z, \lambda) - h(x_2, z, \lambda)| \\
        &\le \sup_{z' \in \cZ} |\ell(x_1, z') - \ell(x_2, z')|
        \le G_\cX \norm{x_1 - x_2}.
    \end{aligned}
    \end{align}
    Next, we bound $|\tilde h(x, z, \lambda)|$ for $\lambda \ge \bar \lambda$.
    If $p = 1$, for any $x \in \cX$, $z \in \cZ$ and $\lambda \ge \bar \lambda = \norm{\ell(x^\star, \cdot)}_{\lip}$, Lipschitz continuity of $\ell(\cdot, z)$ in $x$ gives
    \begin{align*}
        \ell(x^\star, z) - G_\cX D_\cX
        \le \ell(x, z)
        &\le h(x, z, \lambda) \\
        &= \sup_{z' \in \cZ} \left\{\ell(x, z') - \lambda\norm{z' - z}\right\} \\
        &\le \sup_{z' \in \cZ}\left\{\ell(x^\star, z') - \lambda\norm{z' - z}\right\} + G_\cX D_\cX
        = \ell(x^\star, z) + G_\cX D_\cX,
    \end{align*}
    where the last equality uses the choice $z' = z$, which is optimal as $\lambda \ge \norm{\ell(x^\star, \cdot)}_{\lip}$.
    Hence, when $p = 1$, we have the bound
    \begin{align}
    \label{eq:Mz:1}
        |\tilde h(x, z, \lambda)| \le M(z) := G_\cX D_\cX.
    \end{align}
    If $p > 1$, for any $x \in \cX$, $z, z' \in \cZ$, and $\lambda \ge \bar \lambda = 2^p g^\star$, the growth condition $|\ell(x^\star, z')| \le g^\star(1 + \norm{z'}^p)$ and the basic inequality $\norm{z'}^p \le 2^{p-1}\norm{z}^p + 2^{p-1}\norm{z' - z}^p$ yield
    \begin{align*}
        \ell(x, z') - \lambda\norm{z' - z}^p
        &\le \ell(x^\star, z') + G_\cX D_\cX - \lambda\norm{z' - z}^p \\
        &\le g^\star \left(1+\norm{z'}^p\right) + G_\cX D_\cX - \lambda\norm{z' - z}^p \\
        &\le G_\cX D_\cX + g^\star \left(1 + 2^{p-1}\norm{z}^p\right) - \left(\lambda - 2^{p-1}g^\star \right) \norm{z' - z}^p \\
        &\le G_\cX D_\cX + g^\star \left(1 + 2^{p-1}  \norm{z}^p\right).
    \end{align*}
    Taking the supremum over $z' \in \cZ$ yields the upper bound
    \begin{align*}
        h(x, z, \lambda) \le G_\cX D_\cX + g^\star\left(1 + 2^{p-1}\norm{z}^p\right).
    \end{align*}
    On the other hand, choosing $z' = z$ in the definition of $h(x, z, \lambda)$ gives the lower bound
    \begin{align*}
        h(x, z, \lambda) \ge \ell(x, z) \ge \ell(x^\star, z) - G_\cX D_\cX \ge -G_\cX D_\cX - g^\star\left(1 + \norm{z}^p\right).
    \end{align*}
    Combining these two bounds and using $\tilde h(x, z, \lambda) = h(x, z, \lambda)$ shows that, when $p > 1$,
    \begin{align}
    \label{eq:Mz:2}
        |\tilde h(x, z, \lambda)| \le M(z) := G_\cX D_\cX + g^\star\left(1 + 2^{p-1}\norm{z}^p\right).
    \end{align}
    In both cases, $B := \EE_{\bP^\star}[M(z)] < \infty$ by Assumption~\ref{asp::regularity}\,(i).
    
    \paragraph{Step 2 (Discretizing $\cX$):}
    Set $\varepsilon_\cX = D_\cX / T$.
    Since $\cX \subseteq \bR^n$ has diameter $D_\cX$, \citep[Proposition~4.2.12]{vershynin2019high} ensures the existence of an $\varepsilon_\cX$-net $\cN_\cX \subseteq \cX$ with cardinality
    \begin{align}
    \label{eq::covering_number_X}
    |\cN_\cX| \le \left(1 + \frac{2 D_\cX}{\varepsilon_\cX}\right)^n = (1 + 2T)^n.
    \end{align}
    For any $x \in \cX$, there exists $\hat{x} \in \cN_\cX$ such that $\norm{x - \hat{x}} \le D_\cX / T$.
    By \eqref{eq::h_tilde_lip_x}, for every $t \le T$ and $\lambda \ge \bar \lambda$,
    \begin{align*}
        |\Delta_t(x, \lambda) - \Delta_t(\hat{x}, \lambda)|
        &\le \EE_{\bP^\star}\big[|\tilde h(x, z, \lambda) - \tilde h(\hat{x}, z, \lambda)|\big]
        + \EE_{\hat{\bP}_t}\big[|\tilde h(x, z, \lambda) - \tilde h(\hat{x}, z, \lambda)|\big] \\
        &\le 2 G_\cX \norm{x - \hat{x}}
        \le \frac{2 G_\cX D_\cX}{T}.
    \end{align*}
    Consequently, we have
    \begin{align}
    \label{eq::net_reduction_X}
        \sup_{x \in \cX,\, \lambda \ge \bar \lambda} \Delta_t(x, \lambda)
        \le \max_{\hat{x} \in \cN_\cX} \sup_{\lambda \ge \bar \lambda} ~ \Delta_t(\hat{x}, \lambda) + \frac{2 G_\cX D_\cX}{T}.
    \end{align}
    
    \paragraph{Step 3 (Discretizing $\{ \lambda : \lambda \ge \bar \lambda\}$):}
    Fix $\hat{x} \in \cN_\cX$.
    By the bounds \eqref{eq:Mz:1} and \eqref{eq:Mz:2}, the expectation $\EE_{\bP^\star}[\tilde{h}(\hat{x}, z, \lambda)]$ takes values in $[-B, B]$, where $B := \EE_{\bP^\star}[M(z)] < \infty$.
    Moreover, for each $z \in \cZ$, the function $\lambda \mapsto \tilde{h}(\hat{x}, z, \lambda)$ is convex and nonincreasing on $[\bar{\lambda}, \infty)$ and decreases to $\tilde{h}(\hat{x}, z, \infty)$ as $\lambda \to \infty$, where $h(\hat{x}, z, \infty) := \ell(\hat{x}, z)$.
    Hence, $\EE_{\bP^\star}[\tilde{h}(\hat{x}, z, \lambda)]$ is continuous and nonincreasing on $[\bar{\lambda}, \infty]$, and its range has length at most $2B$.
    Partitioning \emph{this range} into intervals of length at most $B / T$, we obtain a finite grid of multipliers
    \begin{align*}
        \Lambda(\hat{x}) = \{\lambda_0, \lambda_1, \dots, \lambda_J\}
        \quad \text{with} \quad
        \bar \lambda = \lambda_0 < \lambda_1 < \cdots < \lambda_J = \infty
        \quad \text{and} \quad
        J \le 2T + 1,
    \end{align*}
    such that for every $j \in [J]$, we have
    \begin{align*}
        0 \le \EE_{\bP^\star}[\tilde h(\hat{x}, z, \lambda_{j-1})] - \EE_{\bP^\star}[\tilde h(\hat{x}, z, \lambda_j)] \le \frac{B}{T}.
    \end{align*}
    For any $\lambda \ge \bar \lambda$, choose $j \in [J]$ such that $\lambda \in [\lambda_{j-1}, \lambda_j]$.
    Since $\tilde h(\hat{x}, z, \lambda_j) \le \tilde h(\hat{x}, z, \lambda) \le \tilde h(\hat{x}, z, \lambda_{j-1})$ for all $z \in \cZ$, we have
    \begin{align*}
        \Delta_t(\hat{x}, \lambda)
        &= \EE_{\bP^\star}[\tilde h(\hat{x}, z, \lambda)] - \EE_{\hat{\bP}_t}[\tilde h(\hat{x}, z, \lambda)] \\
        &\le \EE_{\bP^\star}[\tilde h(\hat{x}, z, \lambda_{j-1})] - \EE_{\hat{\bP}_t}[\tilde h(\hat{x}, z, \lambda_j)] \\
        &= \Delta_t(\hat{x}, \lambda_j) + \left(\EE_{\bP^\star}[\tilde h(\hat{x}, z, \lambda_{j-1})] - \EE_{\bP^\star}[\tilde h(\hat{x}, z, \lambda_j)]\right) \le \Delta_t(\hat{x}, \lambda_j) + \frac{B}{T}.
    \end{align*}
    Combining this with \eqref{eq::net_reduction_X} gives, for every $t \le T$,
    \begin{align}
    \label{eq::full_discretization}
    \sup_{x \in \cX,\, \lambda \ge \bar \lambda} \Delta_t(x, \lambda)
    \le \max_{\hat{x} \in \cN_\cX} \max_{\lambda_j \in \Lambda(\hat{x})} \Delta_t(\hat{x}, \lambda_j)
    + \frac{2 G_\cX D_\cX + B}{T}.
    \end{align}
    
    \paragraph{Step 4 (Concentration and union bound):}
    By \eqref{eq::covering_number_X} and $|\Lambda(\hat{x})| \le 2T + 2$, the total number of elements $(t, \hat{x}, \lambda_j)$ with $t \in [T]$, $\hat{x} \in \cN_\cX$, and $\lambda_j \in \Lambda(\hat{x})$ is at most
    \begin{align*}
    N_T := T (1 + 2T)^n (2T + 2) \le (2T + 2)^{n+2}.
    \end{align*}
    If $p = 1$, inequality~\eqref{eq:Mz:1} implies $|\tilde h(\hat{x}, z, \lambda_j)| \le G_\cX D_\cX$ almost surely.
    By Hoeffding's inequality~\citep[Theorem~2.2.6]{vershynin2019high}, for any fixed $(t, \hat{x}, \lambda_j)$ and $u > 0$,
    \begin{align*}
        \bP^\star\big(\Delta_t(\hat{x}, \lambda_j) > u\big)
        \le \exp\left(-\frac{t u^2}{2 G_\cX^2 D_\cX^2}\right).
    \end{align*}
    Setting the right-hand side equal to $\tau / N_T$ and applying the union bound over all at most $N_T$ triples shows that, with probability at least $1 - \tau$, simultaneously for all $t \in [T]$, $\hat{x} \in \cN_\cX$ and $\lambda_j \in \Lambda(\hat{x})$,
    \begin{align*}
        \Delta_t(\hat{x}, \lambda_j)
        \le G_\cX D_\cX \sqrt{\frac{2 \log(N_T / \tau)}{t}}
        \le C' G_\cX D_\cX \sqrt{\frac{\log(2T / \tau)}{t}},
    \end{align*}
    where $C'$ depends only on $n$.
    Substituting this estimate and $B = G_\cX D_\cX$ into \eqref{eq::full_discretization}, and absorbing $3 G_\cX D_\cX / T$ into the leading term, proves part~(i).
    
    If $p > 1$, for any $a > p$, we have $\norm{z}^p \le 1 + \norm{z}^a$ for all $z \in \cZ$. 
    Hence, Assumption~\ref{asp::light_tailed} gives 
    \begin{align*}
        \EE_{\bP^\star}[\exp(\norm{z}^p)] \le \EE_{\bP^\star}[\exp(1+\norm{z}^a)] < \infty.
    \end{align*}
    Inequality~\eqref{eq:Mz:2} then implies that $\tilde{h}(\hat{x}, z, \lambda_j)$ is subexponential under $\bP^\star$ with subexponential norm bounded by a constant $K_0 > 0$ depending only on $p$, $G_\cX D_\cX$, $g^\star$, and $\EE_{\bP^\star}[\exp(\norm{z}^a)]$.
    By Bernstein's inequality~\citep[Theorem~2.8.1]{vershynin2019high}, for any fixed $(t, \hat{x}, \lambda_j)$ and $u > 0$, we have
    \begin{align*}
        \bP^\star \big(\Delta_t(\hat{x}, \lambda_j) > u\big)
        \le \exp\left(-c_0 t \min\left\{\frac{u^2}{K_0^2}, \frac{u}{K_0}\right\}\right),
    \end{align*}
    where $c_0 > 0$ is an absolute constant.
    Setting the right-hand side equal to $\tau / N_T$ and applying the union bound over all at most $N_T$ triples shows that, with probability at least $1 - \tau$, simultaneously for all $t \in [T]$, $\hat{x} \in \cN_\cX$, and $\lambda_j \in \Lambda(\hat{x})$, we have
    \begin{align*}
        \Delta_t(\hat{x}, \lambda_j)
        \le C' K_0 \left(\sqrt{\frac{\log(2T / \tau)}{t}} + \frac{\log(2T / \tau)}{t}\right),
    \end{align*}
    where $C'$ depends only on $n$.
    Substituting this estimate into \eqref{eq::full_discretization} and absorbing $(2 G_\cX D_\cX + B) / T$ into the right-hand side proves part~(ii).
\end{proof}

We are now ready to prove Theorem~\ref{thm::y_best_response}.

\begin{proof}[Proof of Theorem~\ref{thm::y_best_response}] Throughout the proof, we write $C$ for $C_1$, $C_2$, or $C_3$ when the intended constant is clear from context.
    We first bound the second term in Lemma~\ref{lemma::decomposition}.
    By \citep[Theorem~1]{blanchet2019quantifying}, we have
    \begin{align*}
    \sup_{\bQ \in \bB_\rho^p(\bP)} f(x, \bQ)
    = \inf_{\lambda \ge 0}
    \left\{\lambda\rho + \EE_{\bP}[h(x, z, \lambda)]\right\}.
    \end{align*}
    Fix $x \in \cX$, $t \le T$, and $\epsilon > 0$, and choose $\lambda_\epsilon \ge 0$ such that
    \begin{align*}
    \lambda_\epsilon\rho + \EE_{\hat{\bP}_t}[h(x, z, \lambda_\epsilon)]
    \le \sup_{\bQ \in \bB_\rho^p(\hat{\bP}_t)} f(x, \bQ) + \epsilon.
    \end{align*}
    Using the same multiplier $\lambda_\epsilon$ in the problem centered at $\bP^\star$ gives
    \begin{align}
    \label{eq::dual_diff_thm}
    \sup_{\bQ \in \bB_\rho^p(\bP^\star)} f(x, \bQ)
    - \sup_{\bQ \in \bB_\rho^p(\hat{\bP}_t)} f(x, \bQ)
    \le \EE_{\bP^\star}[h(x, z, \lambda_\epsilon)]
    - \EE_{\hat{\bP}_t}[h(x, z, \lambda_\epsilon)] + \epsilon.
    \end{align}
    We bound the difference of expectations uniformly over $x \in \cX$ and all multipliers for which $h(x, \cdot, \lambda)$ is finite.
    In the following, let $\pi \in \Gamma(\bP^\star, \hat{\bP}_t)$ be an optimal coupling for $\W_p(\hat{\bP}_t, \bP^\star)$.
    
    \paragraph{(i) Case $\bm{p = 1}$:}
    If $\lambda \le \norm{\ell(x^\star, \cdot)}_{\lip}$, every finite $h(x, \cdot, \lambda)$ is $\lambda$-Lipschitz as shown in the proof of Lemma~\ref{lem-general}.
    Thus, we have
    \begin{align*}
        \EE_{\bP^\star}[h(x, z, \lambda)] - \EE_{\hat{\bP}_t}[h(x, z, \lambda)]
        &= \EE_\pi[h(x, z_1, \lambda) - h(x, z_2, \lambda)] \\
        &\le \lambda\EE_\pi[\norm{z_1 - z_2}]
        \le \norm{\ell(x^\star, \cdot)}_{\lip} \W_1(\hat{\bP}_t, \bP^\star).
    \end{align*}
    If $\lambda \ge \norm{\ell(x^\star, \cdot)}_{\lip}$, we combine Lemma~\ref{lem::large_lambda}\,(i) with the bound
    \begin{align*}
        \EE_{\bP^\star}[\ell(x^\star, z)] - \EE_{\hat{\bP}_t}[\ell(x^\star, z)]
        \le \norm{\ell(x^\star, \cdot)}_{\lip}\W_1(\hat{\bP}_t, \bP^\star).
    \end{align*}
    Together, the two ranges of $\lambda$ imply that, with probability at least $1 - \tau$, simultaneously for all $t \le T$, $x \in \cX$ and $\lambda \ge 0$,
    \begin{align*}
        \EE_{\bP^\star}[h(x, z, \lambda)] - \EE_{\hat{\bP}_t}[h(x, z, \lambda)]
        \le \norm{\ell(x^\star, \cdot)}_{\lip}\W_1(\hat{\bP}_t, \bP^\star)
        + C G_\cX D_\cX \sqrt{\frac{\log(2T/\tau)}{t}}.
    \end{align*}
    Substituting this into \eqref{eq::dual_diff_thm} at $x = x_t$ and letting $\epsilon \downarrow 0$ bounds the second term in Lemma~\ref{lemma::decomposition}.
    Applying Proposition~\ref{prop::convergence} at $x^\star$ bounds the third term by $\norm{\ell(x^\star, \cdot)}_{\lip}\W_1(\hat{\bP}_t, \bP^\star)$. 
    Substituting all these bounds into Lemma~\ref{lemma::decomposition} concludes the proof of the first claim.
    
    \paragraph{(ii) Case $\bm{p > 1}$:}
    Take $g^\star$ from Assumption~\ref{asp::regularity} at $x^\star$.
    When $\lambda = 0$, a finite $h(x, \cdot, 0)$ is constant and its difference of expectations is zero.
    Suppose first that $0 < \lambda \le 2^p g^\star$.
    For each $z_1 \in \cZ$, choose $z' \in \argmax_{u \in \cZ}\left\{\ell(x, u) - \lambda\norm{u - z_1}^p\right\}$, which is attained because the pieces of $\ell(x, \cdot)$ are concave (have affine upper bounds) and $p > 1$.
    By optimality of $z'$, Lipschitz continuity in $x$ and concavity of the pieces at $x^\star$, we have
    \begin{align*}
        \lambda\norm{z' - z_1}^p
        \le \ell(x, z') - \ell(x, z_1)
        &\le \ell(x^\star, z') - \ell(x^\star, z_1) + 2G_\cX D_\cX \\
        &\le \max_{k \in [K]} \norm{\nabla_z\ell_k(x^\star, z_1)}\norm{z' - z_1} + 2G_\cX D_\cX.
    \end{align*}
    Since the right-hand side is the sum of two nonnegative terms, either 
    \begin{align*}
        \lambda\norm{z' - z_1}^p \le 2\max_{k \in [K]}\norm{\nabla_z\ell_k(x^\star, z_1)}\norm{z' - z_1} \quad \text{or} \quad \lambda\norm{z' - z_1}^p \le 4G_\cX D_\cX.
    \end{align*}
    Dividing by $\norm{z' - z_1}$ in the first case and raising to the power $(p-1)/p = 1/q$ in the second case yields
    \begin{align*}
        \lambda\norm{z' - z_1}^{p-1}
        \le 4\left(\max_{k \in [K]}\norm{\nabla_z\ell_k(x^\star, z_1)} + \lambda^{1/p}(G_\cX D_\cX)^{1/q}\right).
    \end{align*}
    Using $z'$ as a feasible point in the definition of $h(x, z_2, \lambda)$ and the basic inequality $(a + b)^p - a^p \le p 2^{p-1}(a^{p-1}b + b^p)$ for $a, b \ge 0$ yields
    \begin{align*}
        h(x, z_1, \lambda) - h(x, z_2, \lambda)
        &\le \lambda\left(\norm{z' - z_2}^p - \norm{z' - z_1}^p\right) \\
        &\le p 2^{p-1} \lambda\left(\norm{z' - z_1}^{p-1}\norm{z_1 - z_2} + \norm{z_1 - z_2}^p\right).
    \end{align*}
    Using the above bounds and integrating with respect to $\pi$ yields
    \begin{align*}
        &\EE_{\bP^\star}[h(x, z, \lambda)] - \EE_{\hat{\bP}_t}[h(x, z, \lambda)] \\
        &= \EE_\pi[h(x, z_1, \lambda) - h(x, z_2, \lambda)] \\
        &\le p 2^{p+1}\,\EE_\pi\left[\left(\max_{k \in [K]}\norm{\nabla_z\ell_k(x^\star, z_1)} + \lambda^{1/p}(G_\cX D_\cX)^{1/q}\right)\norm{z_1 - z_2}\right]
        + p 2^{p-1}\lambda\,\EE_\pi\left[\norm{z_1 - z_2}^p\right] \\
        &\le p 2^{p+1}\left(\norm{\ell(x^\star, \cdot)}_{\dot{H}^{1,q}(\bP^\star)} + \lambda^{1/p}(G_\cX D_\cX)^{1/q}\right)\left(\EE_\pi\left[\norm{z_1 - z_2}^p\right]\right)^{1/p}
        + p 2^{p-1}\lambda\,\EE_\pi\left[\norm{z_1 - z_2}^p\right] \\
        &\le C\left(\W_p(\hat{\bP}_t, \bP^\star) + \W_p^p(\hat{\bP}_t, \bP^\star)\right),
    \end{align*}
    where the second inequality uses H\"older's inequality and the fact that the first marginal of $\pi$ is $\bP^\star$, and the last inequality uses $(\EE_\pi[\norm{z_1 - z_2}^p])^{1/p} = \W_p(\hat{\bP}_t, \bP^\star)$ and $\lambda \le 2^p g^\star$.
    Together with Lemma~\ref{lem::large_lambda}\,(ii) for $\lambda \ge 2^p g^\star$, the two ranges of $\lambda$ imply that, with probability at least $1 - \tau$, simultaneously for all $t \le T$, $x \in \cX$ and $\lambda \ge 0$,
    \begin{align*}
        \EE_{\bP^\star}[h(x, z, \lambda)] - \EE_{\hat{\bP}_t}[h(x, z, \lambda)]
        \le C\left(
        \W_p(\hat{\bP}_t, \bP^\star)
        + \W_p^p(\hat{\bP}_t, \bP^\star)
        + \sqrt{\frac{\log(2T/\tau)}{t}}
        + \frac{\log(2T/\tau)}{t}
        \right).
    \end{align*}
    Substituting this into \eqref{eq::dual_diff_thm} at $x = x_t$ and letting $\epsilon \downarrow 0$ bounds the second term in Lemma~\ref{lemma::decomposition}.
    Applying Proposition~\ref{prop::convergence} at $x^\star$ bounds the third term by $\norm{\ell(x^\star, \cdot)}_{\dot{H}^{1,q}(\bP^\star)}\W_p(\hat{\bP}_t, \bP^\star)$.
    Since both robust objectives are $G_\cX$-Lipschitz in $x$, the sum of the two summands at round $t$ is also at most $2G_\cX D_\cX$, which absorbs the higher-order terms $\W_p^p(\hat{\bP}_t, \bP^\star)$ and $\log(2T/\tau)/t$.
    Thus, the two summands at round $t$ are together bounded by
    \begin{align*}
        C\W_p(\hat{\bP}_t, \bP^\star)
        + C\sqrt{\frac{\log(2T/\tau)}{t}}.
    \end{align*}
    Substituting all these bounds into Lemma~\ref{lemma::decomposition} concludes the proof of the second claim.
\end{proof}

\subsection{Proof of Corollary~\ref{cor::convergence}}

Applying Theorem~\ref{thm::y_best_response} with failure probability $\tau/2$, we obtain, with probability at least $1 - \tau/2$,
\begin{align}
\label{eq::prob-bound}
\mathrm{Gap}(\bar{x}_T)
&\le \frac{3G_\cX D_\cX}{2\sqrt{T}}
+ \frac{C}{T}\sum_{t=1}^T
\W_p(\hat{\bP}_t, \bP^\star)
+ C\sqrt{\frac{\log(4T/\tau)}{T}}
+ \delta.
\end{align}
Here and below, $C$ denotes positive constants with the dependencies stated in the corollary.

Since $p \neq m/2$, \citep[Theorem~2]{fournier2015rate} and a union bound over $t = 1, \ldots, T$ show that, with probability at least $1 - \tau/2$, we have
\begin{align*}
    \W_p(\hat{\bP}_t, \bP^\star)
    \le
    \begin{cases}
    c\left(\dfrac{\log(T/\tau)}{t}\right)^{1/a},
    & t < c\log(T/\tau), \\[1ex]
    c\left(\dfrac{\log(T/\tau)}{t}\right)^{\min\{1/m,\,1/(2p)\}},
    & t \ge c\log(T/\tau),
    \end{cases}
\end{align*}
simultaneously for all $t \le T$.
The constant $c \ge 1$ depends only on $m$, $p$, $a$, and the exponential moment in Assumption~\ref{asp::light_tailed}.
Choose $C_1 \ge c$.

Since $a > 1$ and $\min\{1/m, 1/(2p)\} \le 1/2$, summing the two ranges separately gives
\begin{align*}
    \frac{1}{T}\sum_{t=1}^T
    \W_p(\hat{\bP}_t, \bP^\star)
    &\le
    \frac{c}{T}
    \sum_{1 \le t < c\log(T/\tau)}
    \left(\frac{\log(T/\tau)}{t}\right)^{1/a} \\
    &\quad +
    \frac{c}{T}
    \sum_{c\log(T/\tau) \le t \le T}
    \left(\frac{\log(T/\tau)}{t}\right)^{\min\{1/m,\,1/(2p)\}} \\
    &\le C\left(
    \frac{\log(T/\tau)}{T}
    + \left(\frac{\log(T/\tau)}{T}\right)^{\min\{1/m,\,1/(2p)\}}
    \right) \\
    &\le C
    \left(\frac{\log(T/\tau)}{T}\right)^{\min\{1/m,\,1/(2p)\}}.
\end{align*}
The second inequality follows by integral comparison for the two sums.
The last uses $\log(T/\tau)/T \le 1$.

Finally, $T \ge 3$ and $C_1 \ge 1$ imply
$1 \le \log(T/\tau) \le T$ and
$\log(4T/\tau) \le 3\log(T/\tau)$.
Together with
$\min\{1/m, 1/(2p)\} \le \min\{p/m, 1/2\} \le 1/2$
and the prescribed value of $\delta$, these inequalities give
\begin{align*}
    \frac{1}{\sqrt{T}}
    + \sqrt{\frac{\log(4T/\tau)}{T}}
    + \delta
    \le C
    \left(\frac{\log(T/\tau)}{T}\right)^{\min\{1/m,\,1/(2p)\}}.
\end{align*}
Substituting into \eqref{eq::prob-bound} proves the claim on the intersection of the two events.
This intersection has probability at least $1 - \tau$ by the union bound.
\hfill \qed

\end{document}